\documentclass{article}
\usepackage[colorlinks=true, linkcolor=blue, citecolor=blue,urlcolor=black]{hyperref}

\usepackage[utf8]{inputenc} 
\usepackage[T1]{fontenc}    
\usepackage{hyperref}       
\usepackage{url}            
\usepackage{booktabs}       
\usepackage{amsfonts}       
\usepackage{nicefrac}       
\usepackage{microtype}      
\usepackage{xcolor}         

\usepackage{algorithm}
\usepackage[noend]{algpseudocode}
\usepackage[utf8]{inputenc}
\usepackage{comment}
\usepackage{amsmath}
\usepackage{amssymb}
\usepackage{amsthm}
\usepackage{geometry}
\usepackage{graphicx}
\graphicspath{ {./images/} }
\usepackage{amsmath,amssymb}
\usepackage{fullpage}
\usepackage{color}   
\usepackage{hyperref}
\hypersetup{
	colorlinks=true, 
	linktoc=all,     
	linkcolor=blue,  
}

\newtheorem{theorem}{Theorem}

\newtheorem{definition}{Definition}

\newtheorem{remark}{Remark}
\newtheorem{lemma}{Lemma}
\newtheorem{claim}{Claim}

\title{Online Learning for Cooperative Multi-Player Multi-Armed Bandits}

\author{%
	William Chang, Mehdi Jafarnia-Jahromi and Rahul Jain \\
	University of Southern California\\
	\texttt{(chan087,mjafarni,rahul.jain)@usc.edu} \\
}
\begin{document}
	
	\maketitle
	
	\begin{abstract}
		We introduce a framework for decentralized online learning for multi-armed bandits (MAB) with multiple cooperative players. The reward obtained by the players in each round depends on the actions taken by all the players. It's a team setting, and the objective is common. Information asymmetry is what makes the problem interesting and challenging. We consider three types of information asymmetry: action information asymmetry when the actions of the players can't be observed but the rewards received are common; reward information asymmetry when the actions of the other players are observable but rewards received are IID from the same distribution; and when we have both action and reward information asymmetry. For the first setting, we propose a UCB-inspired algorithm that achieves $O(\log T)$ regret whether the rewards are IID or Markovian. For the second section, we offer an environment such that the algorithm given for the first setting gives linear regret. For the third setting, we show that a variation of the `explore then commit' algorithm achieves almost log regret. 
	\end{abstract}
	
	\section{Introduction}\label{sec:intro}

Multi-armed bandit (MAB) models are prototypical models for online learning to understand the exploration-exploitation tradeoff. There is huge literature on such models beginning with Bayesian bandit models \cite{gittins2011multi}, and non-Bayesian bandits \cite{lattimore2020bandit} which were introduced by Lai and Robbins in \cite{lai1985asymptotically}. A key algorithm for MAB models is the UCB$_1$ algorithm introduced in \cite{auer2002finite} that spurred a lot of innovation on such  \textit{Optimisim in the face of uncertainty (OFU)} algorithms. This included multiplayer multi-armed bandit models introduced in a matching context in \cite{gai2012combinatorial, liu2010distributed, anandkumar2011distributed, kalathil2014decentralized, nayyar2016regret}. This was motivated by the problem of spectrum sharing in wireless networks wherein the users want to get matched to channels, each channel can be occupied by only one wireless user, and together they want to maximize the expected sum throughput.

In this paper, we  consider a general multi-agent multi-armed bandit model wherein each agent has a different set of arms. The players have a team (common) objective, and the individual rewards depend on the arms selected by all the players. This can be seen as a multi-dimensional version of a single agent MAB model, and still the same centralized algorithm should work. The twist is that there is \textit{information asymmetry} between the players. This can be of various types. First, a player may not be able to observe the actions of the other players (action information asymmetry). Thus, despite the reward being common information, each agent is not really able to tell what arm-tuple the reward came from. Second, the players may observe the actions of the other players, but they can get different i.i.d. rewards corresponding to the distribution of the arm-tuple played (reward information asymmetry). Third, we may have both action and reward information asymmetry between the agents. The two types of information asymmetry  make the problem of decentralized online learning in multi-player MAB models  challenging. Furthermore, as we will see, the three problems are of increasing complexity. The rewards may be i.i.d., i.e., every time an arm-tuple is chosen, players get a i.i.d. random reward from the corresponding distribution. Or they could be Markovian, i.e., the rewards come from a Markov chain corresponding to the arm-tuple chosen. Such a Markov chain only evolves when the corresponding arm-tuple is chosen. 

There have been a number of papers on decentralized MAB models beginning with \cite{kalathil2014decentralized}, which gave the first sublinear (log-squared) regret algorithm, as well as \cite{nayyar2016regret} which gave the first decentralized algorithm with order-optimal regret, followed by a number of others such as \cite{avner2014concurrent} and  \cite{bistritz2018distributed}. All of these results are really about the matching bandits problem, i.e., the players in a decentralized manner want to learn the optimal matching. The setting in this paper is different, and in some sense more general. It is akin to players playing a game with the caveat that the reward depends on actions taken by all of them but is common (or from the same distribution). 

We first consider the problem with action information asymmetry. We define a standard index function for each player, and propose a variant of the UCB$_1$ algorithm we call \texttt{mUCB}. We show that it achieves $O(\log T)$ regret when the gap between the means for the arms is known, and it gets $O(\sqrt{T})$ when the gap is not known.   Next, we consider the problem when there is both action and reward information asymmetry, and show that the \texttt{mDSEE} algorithm is able to achieve near log  regret bound for this setting. Finally next consider the problem with reward information asymmetry wherein actions are observable but rewards to various players are IID (though come from the same distribution), and provide an environmennt such that \texttt{mUCB} gives linear regret. However, as this is a special case of the setting with both action and reward information asymmetry, it is clear that \texttt{mDSEE} will obtain the same order regret in this problem.

\paragraph*{Related Work.} We now discuss related work in more detail.

The literature on multi-armed bandits is overviewed in \cite{sutton2018reinforcement,gittins2011multi,lattimore2020bandit}. Some classical papers worth mentioning are \cite{lai1985asymptotically,anantharam1987bandits,auer2002finite}. Interest in multi-player MAB models was triggered by the problem of opportunistic spectrum sharing and some early papers were \cite{gai2012combinatorial,liu2010distributed,anandkumar2011distributed}. Other papers motivated by similar problems in communications and networks are \cite{maghsudi2014channel,korda2016distributed,shahrampour2017multi,chakraborty2017coordinated}. These papers were either for the centralized case, or considered the symmetric user case, i.e., all users have the same reward distributions. Moreover, if two or more users choose the same arm, there is a ``collision", and neither of them get a positive reward.  The first paper to solve this matching problem in a general setting was \cite{kalathil2014decentralized} which obtained log-squared regret. It was then improved to log regret in \cite{nayyar2016regret} by employing a posterior sampling approach. These algorithms required implicit (and costly) communication between the players. Thus, there were attempts to design algorithms without it \cite{avner2014concurrent, rosenski2016multi, bistritz2018distributed, feraud2019decentralized}.  \cite{besson2018multi} considers the lower bound on regret of the certain for multiplayer MAB algorithms. Other recent papers on decentralized learning for multiplayer matching MAB models are \cite{wang2020optimal, shi2020decentralized}.

Another related thread is the learning and games literature \cite{young2004strategic, fudenberg1998theory, cesa2006prediction}. This literature deals with a strategic learning wherein each player has its own objective, and the question whether a learning process can reach some sort of an equilibrium. Some key results here are \cite{hart2003uncoupled} which showed that learning algorithms with uncoupled dynamics do not lead to a Nash equilibrium. It was also shown that simple adaptive procedures lead to a correlated equilibrium \cite{hart2000simple, hart2001general, hart2001reinforcement, hart2013simple}.  In fact, the model in this paper may be regarded as more closely related to this literature than the multiplayer/decentralized matching MAB models discussed above except that all players have a common objective, a l\'a team theory \cite{marschak1972economic}. 

In the realm of multiplayer stochastic bandits, nearly all works allow for limited communication such as those in \cite{martinez2018decentralized, martinez2019decentralized, szorenyi2013gossip, karpov2020collaborative, tao2019collaborative}. An exception is \cite{bistritz2021one} where all the players select from the same set of arms and their goal is to avoid a collision, that is, they do not want to select the same arm as another player. Another exception is \cite{xu2015distributed} where they developed  online learning algorithms that enable agents to cooperatively learn how to maximize reward with noisy global feedback without exchanging information.

	\section{Preliminaries and Problem Statements}\label{sec:prelim}

\subsection{Multi-player MAB model with IID Rewards} 
\paragraph*{Problem A: Action Information Asymmetry with unobserved actions, common rewards.}
We first present a multi-player multi-armed bandit (MMAB) model, and then a series of problem formulations. Consider a set of $M$ players $P_1,\cdots,P_M$, in which player $P_i$ has a set ${\mathcal K}_i$ of $K_i$ arms to pick from. At each time instant, each player picks an arm from their set with the $M$-tuple of arms picked denoted by $a=(a_1,\cdots,a_M)$. This generates a random reward $X_{a} \in [0,1]$ (without loss of generality for bounded reward functions) from a $1$-subgaussion distribution $F_{a}$  with mean $\mu_{a}$. Denote $\Delta_{a} = \mu^* - \mu_{a}$ where $\mu^*$ is the highest reward mean among all arm tuples. Let $a_i(t) \in \{1,\cdots,K_i\}$ be the arm chosen by player $i$ at time $t$, and denote $a(t) = (a_1(t),\cdots,a_M(t))$. If arms $a(t)$ are pulled by the players collectively, the reward to all the players  equals $X_{a(t)}(t)$, i.e., the reward is common, depends on $a(t)$ and is independent across time instants. A high-level objective is for the players to collectively identify the best set of arms $a^*$ corresponding to mean reward $\mu^*$. But the players do not know the means $\mu_a$, nor the distributions $F_a$. They must learn by playing and exploring. Thus, we can capture learning efficiency of an algorithm via the notion of per player \textit{expected regret},
\begin{equation}\label{eq:regret}
	R_T = \mathbb{E}\left[T\mu^* - \sum_{t=1}^T X_{a(t)}(t)\right]
\end{equation}
where $T$ is the number of learning instances  and $X_{a(t)}(t)$ is the random reward if arm-tuples $a(t)$ are pulled.  Note that if the players jointly learn the optimal arms $a^*$ eventually, per unit expected regret $R_T/T \to 0$ as $T \to \infty$. Thus, our goal is to design a multi-player decentralized learning algorithm that has sublinear (expected) regret. Note that fundamental results for single-player MAB problems \cite{lai1985asymptotically} suggest a $O(\log T)$-regret lower bound for the multi-player MAB problem as well. If we can design a multi-player decentralized learning algorithm with such a regret order, then it would imply that such a lower bound is tight for this setting as well. A multi-player decentralized learning for an MAB problem would essentially be like a single-player (centralized) learning for an MAB problem with multi-dimensional arms if all players have the same information while making decisions at all times. But this may not hold in various applications as various players may not be able to observe for example, actions of the other players. Thus, there may be \textit{information asymmetry} between the players, which makes the problem much more challenging. We specifically exclude any explicit communication between the players during learning though they may coordinate \textit{a priori.} 

\textit{Markovian Reward model for Problem A.} Finally, we will consider rewards to be Markovian, i.e., for a fixed tuple of arms $a$, the reward sequence $(X^i_a(t))$ for a fixed $i$ is a Markov chain. And moreover, the Markov chains for $i=1,\cdots,M$ are independent and identical. The notion of per-player expected regret would be the same as defined in \eqref{eq:regret}. 

\paragraph*{Problem B': Reward Information Asymmetry with observed actions, independent rewards.} We will see this as a special case of Problem B below.
Here we consider the setting where each agent can observe the actions of all the other agents, but the rewards to each player are different, i.e., if the arms pulled at time $t$ are $a(t)$, each player gets an i.i.d. copy of the rewards from the distribution $F_{a(t)}$, i.e., player $i$ gets $X^i_{a(t)}(t)$. Thus, the  expected regret of player $i$ is obtained by replacing $X_{a(t)}(t)$ in \eqref{eq:regret} with $X^i_{a(t)}(t)$. But note that the expected regret would be the same for all players since the rewards obtained are IID.
All of the players $P_1,\cdots, P_M$ can see what arms the other players are pulling, however, the rewards are identically and independently generated for both players. This means each of the players would end up seeing different reward realizations but from the same distributions.

\paragraph*{Problem B: Action and Reward Information Asymmetry with unobserved actions, independent rewards.}
Next we can have asymmetry in information in terms of action observations, as well as rewards. We will assume that no agent can observe the actions of the others, as well as the rewards to each player are independent and from the same distribution $F_{a(t)}$ if the tuple of actions taken at time $t$ is $a(t)$. In this case, the expected regret for each player is still defined as in \eqref{eq:regret}. Note that Problem B is more general Problem B'.

\paragraph*{Our Contributions.}
For Problem A, we provide a simple and elegant method method for multiple players to coordinate their arms only knowing the global rewards unseen in previous literature. On Problem B, we are able to provide an algorithm that provides a gap-independent algorithm which attains the almost optimal near-log regret, whereas the disCo algorithm in \cite{xu2015distributed} is gap-dependent. For special case Problem B', we give a counterexample which shows that Algorithm \ref{algo:mucb} designed for Problem A obtains linear regret even when the players can see the other player's arms. 


	\section{Online Multi-player Bandit Learning:  Algorithms and Main Results}\label{sec:algos}

We now present  algorithms and their regret performance for the various settings introduced above. 

\subsection{Problems A: IID and Markovian Rewards}\label{sec:algo-probA} 

Let us first define a UCB (Upper Confidence Bound) index to be used by each player. Suppose at time $t$, player $i$ chooses arm (action) $a_i$. Then, the UCB index for player $i$ is given by
\begin{equation}\label{index:mucb}
	\eta^i_{a}(t) =  \begin{cases}
		\infty, & \text{if } n_a(t) = 0,\\
		\hat{\mu}^i_{a}(t, n_a(t)) + \sqrt{\frac{2\log(1/\delta)}{n_a(t)}}, & \text{otherwise}.
	\end{cases}
\end{equation}
$a = (a_1,\cdots,a_M)$, $\hat{\mu}^i_{a}(t,n_a(t))$ is the average of the rewards received by player $i$ from arm $a$ after $t$ total rounds, and $n_a(t)$ is the number of times arm $a$ has been selected in $t$ rounds. $\delta$ is a confidence parameter which shall be selected later. Note that in the setting here, since all players receive a common reward, $\hat{\mu}^i_a(t,n_a)$ and thus the indices  $\eta^i_a(t)$ are the same for all players. Further note that in this setting, while each player only gets to observe the common reward but not the actions of the player, the players can agree that in the first $K_{\max} = K_1\cdots K_M$ times, the players will play a predetermined sequence of actions such that they together  have at least one reward from each arm-tuple $a = (a_1,\cdots, a_M)$ where $a_i =1,\cdots,K_i$. This initial exploration sequence is similar to that in a single player \texttt{UCB}$_1$ algorithm. At the end of the initial round, all players have received the same reward observations and also know the arm-tuple for each such reward from which they come. Thus, the index computed by each player $i$ for any arm-tuple $a$ is the same.  Thus, while say player $i$ cannot observe the actions of the other players but since they all use the same algorithm and the information available to all is the same, they can anticipate the actions that will be taken by the other players. Thus, at each time instant, index \eqref{index:mucb} can be updated as a function of both action $a_i$ of player $i$ but also actions of the other players. The only difficulty arises if two indices for two arm-tuples say $a$ and $a'$ is the same, in which case they must all break the tie in the same way. Thus, to state the \texttt{mUCB} algorithm, we  define a total order on a set on $\mathbb{R}^m$ as follows. 
\begin{definition}\label{def:order}
	We say that $(x_1,\cdots,,x_M) < (y_1,\cdots,y_M)$ if and only if there exists an $n$ such that $\forall i < n$, $x_i = y_i$, and $x_n < y_n$. 
\end{definition}

We are now ready to define the multi-player UCB algorithm as follows:
\begin{algorithm}[H]
	\caption{\texttt{mUCB} algorithm}
	\begin{algorithmic}[1]
		\State \textbf{Input:} $(\mathcal{K}_1,\cdots,\mathcal{K}_M)$. 
		
		\For{$t \leq K_{\max}$}
		\State 	Player $P_i$ will start from his arm $1$ and successively pull each arm $K_{i+1}\cdots K_{M}$ times before moving to the next arm. They will repeat this entire epoch $K_1\cdots K_{i-1}$ times. 
		\EndFor
		
		\For{$t > K_{\max}$}
		\State Player $P_i$ chooses arm $a_i ^*(t)= \arg \max_{a_i} \left(\max_{a_{-i}}\eta^i_{a_i,a_{-i}}(t)\right)$ in \eqref{index:mucb} which corresponds to player $i$ picking the $i$th component of $a^*(t)$ that maximises the index $\eta_a(t)$.  In case of a tie between say $a$ and $a'$, they pick corresponding components of $a$ such that $a < a'$, where the order relation is as specified in Definition \ref{def:order}.
		
		\State Player $P_i$ updates the UCB index $\eta^i_a(t+1)$ for arm $a$ setting $\delta = \frac{1}{t^2}$ with the received reward $X^i_{a^*(t)}(t)$.
		\EndFor
	\end{algorithmic}
	\label{algo:mucb}
\end{algorithm}

\begin{remark}
	(i) We note that the index \eqref{index:mucb} and Algorithm \ref{algo:mucb} are used for both cases of IID and Markovian rewards though the regret bounds will differ as will see.
	(ii) While each player can only observe their own actions, the reward obtained depends on actions taken by all of them together and is common. However, because the players all have the same observations, due to the coordination in the initial exploration round, they can anticipate actions taken by the others, as well as compute the indices as function of the actions taken by all of them. Thus, despite not being able to observe the actions of the others, nor being able to communciate among themselves, they are still able to coordinate their actions. 
\end{remark}



Thus, we can prove the following upper bounds on gap-dependent expected regret for each player.
\begin{theorem}\label{thm:problemA}
	If each player uses the algorithm \texttt{mUCB} in the setting of Problem A, and the rewards are IID, then the expected regret  $R_T$ is upper bounded as:
	\begin{equation}\label{eq:regretB-gapdep}
		R_T \leq 3 \sum_{a}\Delta_{a} +  \sum_{a: \Delta_a>0}\frac{\left(6+4\sqrt{2}\right)\log T}{\Delta_a}.
	\end{equation} 
\end{theorem}
The proof can be found in Section \ref{sec:probAproofs}. Note that we can anticipate a fundamental lower bound for problem A to be at least $O(\log T)$ (since that is a fundamental lower bound for the single player MAB problem). The upper bound actually matches this. Thus, we can conclude that $O(\log T)$ is an achievable lower bound for problem A.

We next present a gap-independent upper bound on expected regret for problem A. 

\begin{theorem}\label{thm:problemA-gapindep}
	If each player uses the algorithm \texttt{mUCB} in the setting of Problem A, and the rewards are IID,  then the (gap-independent) expected regret  $R_T$ is upper bounded as:	
	\begin{equation}
		R_T \leq 3K_1\cdots K_M + \left(1+\sum_{a: \Delta_{a} > \epsilon}\left(6+4\sqrt{2}\right)\right)\sqrt{T \log T}. 
	\end{equation}
\end{theorem}
The proof can be found in Appendix \ref{sec:appendprobAproofs}. This also matches the known gap-independent fundamental lower bounds on expected regret for the single player MAB problem (and hence also the multi-player MAB problem A). Next, we present regret bounds for the Markovian rewards case. 

\begin{theorem}\label{thm:problemA-markov}
	If each player uses the algorithm \texttt{mUCB} in the setting of Problem A, and the rewards are Markovian,  then the expected regret  $R_T$ is upper bounded as:	
	\begin{equation}
		R_T \leq \sum_{a} \Delta_a\left(2C' + \alpha \log T \right) = O\left(\log T \right)
	\end{equation}
	for some universal constant $C'$. 
\end{theorem}
The proof can be found in Appendix  \ref{sec:appendprobAproofs}.

\subsection{Problem B': IID  Rewards}\label{sec:algo-probB} 

We now consider a different type of information asymmetry, namely actions of other players are observable but each gets an IID copy of the rewards. As stated before, this is a special case of Problem B. The  UCB (Upper Confidence Bound) index to be used by each player in this setting is the same as given in \eqref{index:mucb}.
Note that since actions of other players can be observed, each player can compute the index $\eta^i_a$ for each arm-tuple $a$ though unlike in problem A, these indices will have different values since the rewards are different (and a player knows only its own rewards). The algorithm we use for problem B is still the \texttt{UCB} algorithm given as Algorithm \ref{algo:mucb} except that the actions are directly observable.



\begin{remark}
	(i) We note that the index \eqref{index:mucb} and Algorithm \ref{algo:mucb} are used for both cases of IID and Markovian rewards though the regret bounds will differ as will see.
	(ii) While each player can only observe actions of others, their rewards are different (and IID). The indices they compute will also be different. Thus, unlike in Problem A, they are not able to coordinate their actions which makes problem B much more challenging.
\end{remark}

The following result shows that \texttt{mUCB} no longer produces $O(\log(T))$ regret for this problem.

\begin{theorem}\label{thm:problemB}
	If each player uses the algorithm \texttt{mUCB} in the setting of Problem B, and the rewards are IID,  then there exists a multiplayer multi armed bandit setting such that this algorithm gives linear regret. 
\end{theorem}
The proof can be found in Section \ref{sec:probBproofs}. 

The above negative result motivates us to consider a different algorithm for Problem B (which is a generalization of Problem B'). 

\subsection{Problem B: Unobserved Actions and Independent (IID) rewards}\label{sec:algo-probC}

We introduce the \texttt{DSEE} algorithm inspired by \cite{vakili2013deterministic} for Problem B which gives $O(K(T)\log(T))$ regret for some monotonic $K(T)$ going to infinity. This algorithm is nearly optimal, and gives the same regret for Problem B.   We consider the setting where there is two types of information asymmetry between the players: actions of others are unobserved and the rewards are independent. This problem is a lot more challenging than problem A, and variation of the \texttt{mUCB} algorithm does not work. Thus, we provide the \texttt{DSEE} algorithm inspired by \cite{vakili2013deterministic} for the setting in problem C.

Let the reward for player $i$ at time $t$ on playing arm $a_i(t)$ while the other players play arms $a_{-i}(t)$ be denoted by $X^i_{a(t)}(t)$. Then, the sample mean at time $t$ of rewards by playing arm-tuple $a$ for player $i$ is given by $\mu_a^i(t)$. 

\begin{algorithm}[H]\label{algo:DSEE}
	\caption{\texttt{mDSEE} Algorithm}
	\begin{algorithmic}[1]
		\State \textbf{Input:} $(\mathcal{K}_1,\cdots,\mathcal{K}_M)$. 
		
		\State Pick a monotonic function $K(\lambda): \mathbb{N} \rightarrow \mathbb{N}$ such that $\lim_{t \rightarrow \infty}K(\lambda) = \infty$. 
		
		\State First let $\lambda = 1$
		
		\State Player $P_i$ will start from his arm $1$ and successively pull each arm $K(\lambda)K_{i+1}\cdots K_M$ before moving to the next arm. He will repeat this entire epoch $K_1\cdots K_{i-1}$ times.
		
		\State Player $P_i$ will calculate the sample mean $\mu_j^i(t)$  of the rewards they see for each $M$-tuple $a$ of arms. 
		
		\State Player $P_i$ will choose the $M$-tuple arm with the highest sample mean and commit to his corresponding arm for up until the next power of $2$. In case of a tie, pick randomly. 
		
		\State When $T = 2^n$ for some $n \geq \lfloor\log_2(K(1)K_1\cdots K_M)\rfloor+1$, Player $P_i$ with repeat steps (3) - (5), incrementing $\lambda$ by $1$. 
	\end{algorithmic}
\end{algorithm}
Note that since Problem B' is a special case of Problem B, this algorithm can be also be used for Problem B'. 

The players have an  exploration phase wherein each arm $a$ will be pulled $K(1)$ times before moving on to the next one. After all the arms have been pulled $K$ times they will pick the arm with the highest average reward and commit until the next power of $2$. $K(t)$ is chosen to be a function that goes to $\infty$ to compensate for the expoentially increasing intervals of commitment to a single arm. This process is rigorously justified in the proof of Theorem \ref{thm:problemC}. Since the rewards are independently generated for each player, it is possible that each player obtain a different $M$-tuple optimal arm. However, the probability of this happening is lows since $K(\lambda)$ goes to $\infty$. Since we are exploring at powers of $2$, the regret is $O(K(\lfloor\log_2(T)\rfloor)\log(T)$.

\begin{remark}
	The slower $K(\lambda)$ grows, the better the regret. 
\end{remark}

The \texttt{mDSEE} algorithm achieves the following (gap-independent) expected regret on Problem B (and B').
\begin{theorem}\label{thm:problemC}
	If each player uses the \texttt{mDSEE} algorithm in the setting of Problem B (and B'), and the rewards are IID,  then the (gap-independent) expected regret  $R_T^i$  satisfies $O(K(\lfloor\log_2(T)\rfloor\log(T))$.
\end{theorem}
The proof can be found in Section \ref{sec:probCproofs}.
	\section{Regret Analysis: Proofs of Theorems}\label{sec:analysis}

We need Lemma \ref{corollary5.5} for subgaussian random variables reproduced in Appendix \ref{sec:appendix}. We first present the regret decomposition lemma.
\begin{lemma}
	With regret defined in \eqref{eq:regret}, we have the following regret decomposition for each player:
	\begin{equation}\label{eq:regret_decomp}
		R_T^i = \sum_{a} \Delta_{a}\mathbb{E}\left[n_a(T)\right],
	\end{equation}
	where $n_a(T)$ is the number of times the arm $a$ has been pulled up to round $T$.
	\label{lem:regret_decomp}
\end{lemma} 	
The proof can be found in Appendix \ref{sec:appendix}.

\subsection{Problem A: Unobserved Actions, Common Rewards}\label{sec:probAproofs}

Now, we present proofs of theorems for problem A.

\begin{proof} (Theorem \ref{thm:problemA})
	We suppose that the first arm is the optimal one for each player, that is arm $\left(1,...,1\right)$ has the highest average reward. For any arm $a \in \times_{i=1}^M\mathcal{K}_i$, we define the following ``good" event:
	\begin{equation}\label{eq_good}
		G_{a} =  \left\{\mu_{\left(1,...,1\right)} < \min_{t \in [n]} \eta^i_{\left(1,...,1\right)}\left(t\right)\right\} \bigcap \left\{\hat{\mu}_{a}\left(\cdot, m_a\right) + \sqrt{\frac{2}{m_a}\log\left(\frac{1}{\delta}\right)} < \mu_{\left(1,...,1\right)} \right\},
	\end{equation}
	The set above is the event that the common average reward $\mu_{\left(1,...,1\right)}$ from arm $(1,...,1)$ is not underestimated by the UCB index, and that the upper confidence bound of arm $a$ is below $\mu_{\left(1,...,1\right)}$ after $m_a$ observations. The constant $m_a$ will be chosen later. We show that the regret whether $G_a$ occurs or not is small. Letting $I$ be the indicator function we obtain:
	\begin{equation}\label{3}
		\mathbb{E}\left[n_a(T)\right] = \mathbb{E}
		\left[I\left\{ G_{a}\right\}n_a(T)\right] + \mathbb{E}\left[I\left\{  G_{a}^c\right\}n_a(T)\right].
	\end{equation}
	
	We need the following lemma whose proof can be found in Appendix \ref{sec:appendprobAproofs}.
	
	\begin{lemma}\label{lem:pulls}
		When $G_a$ holds, $n_a(T) \leq m_a$. 
	\end{lemma}

	Thus, using equation \eqref{3} we have the bound:
	\begin{equation}\label{eq:expectation}
		\mathbb{E}\left[n_a(T)\right] \leq m_a+P\left(G_{a}^c\right)n.
	\end{equation}
	
	We will now bound $\mathbb{E}\left[I\left\{  G_{a}^c\right\}n_a(T)\right]$. Since $n_a(T) \leq n$, we have the bound:
	\begin{equation}
		\mathbb{E}\left[I\left\{G_{a}^c\right\}n_a(T)\right] \leq \mathbb{E}\left[I\left\{G_{a}^c\right\}T\right]= nP\left(G_{a}^c\right).
	\end{equation}
	
	We have the following claim for $P\left(G_{a}^c\right)$ whose proof can be found in Appendix \ref{sec:appendprobAproofs}.
	
	\begin{claim}\label{claim:1}
		$P\left(G_{a}^c\right)$ where $G_a$ is as defined as in equation \eqref{eq_good} satisfies:
		\begin{equation}
			P\left(G_a^c\right) \leq e^{-\frac{m_ac^2\Delta_a^2}{2}} + T\delta.
		\end{equation}
		
		where $c$ satisfies $\Delta_a - \sqrt{\frac{2\log\left(1/\delta\right)}{m_a}} \geq c \Delta_a.$
	\end{claim}

	It follows from Claim \ref{claim:1} and equation \eqref{eq:expectation} that 
	\begin{equation}\label{blind.eq4}
		\mathbb{E}\left[n_a(T)\right] \leq m_a + T\left(T\delta +e^{-\frac{m_ac^2\Delta_a^2}{2}}\right).
	\end{equation}
	
	Letting $\delta = \frac{1}{T^2}$, we will select $c$ such that inequality \eqref{blind.eq4} is optimized under the constraints given in Claim \ref{claim:1}. We first set
	$m_a = \left\lceil \frac{4\log(T)}{\left(1-c\right)^2\Delta_{a}^2}\right\rceil$ so that plugging this into inequality \eqref{blind.eq4}, we have 
	\begin{equation}
		\mathbb{E}\left[n_a(T)\right] \leq \left\lceil \frac{4\log(T)}{\left(1-c\right)^2\Delta_{a}^2}\right\rceil + 1+  T^{1-\frac{2c^2}{\left(1-c\right)^2}}.
	\end{equation}
	To minimize regret, we want $1-\frac{2c^2}{\left(1-c\right)^2}\leq 0$ which gives us $c \geq -1+\sqrt{2}$ since $c \in \left(0, 1\right)$. The term $\left\lceil\frac{4\log(T)}{\left(1-c\right)^2\Delta_{a}^2}\right\rceil$ is minimized when $c$ is as small as possible so let us pick $c = -1+ \sqrt{2}$. This gives us the inequality:
	\begin{equation}
		\mathbb{E}\left[n_a(T)\right] \leq \left\lceil \frac{4\log(T)}{\left(2-\sqrt{2}\right)^2\Delta_{a}^2}\right\rceil + 2\leq  \frac{2\left(3+2\sqrt{2}\right)\log(T)}{\Delta_{a}^2} + 3.
	\end{equation}
	Plugging this back into inequality \eqref{eq:regret_decomp} yields the desired result in Theorem \ref{thm:problemA}. 
\end{proof}


\subsection{Problem B': Observed Actions, Independent Rewards}\label{sec:probBproofs}


\begin{proof} \emph{(Theorem \ref{thm:problemB})}
Consider the two player setting where each player has $2$ arms with means as in Figure \ref{fig:counter_ex}. Furthermore, suppose for a square with mean $\mu$, the distribution is uniform across the interval $[\mu - .05, \mu+.05]$. Clearly such distributions are $1$-subgaussian. The row correspond to the arms Player $1$ while the column correspond to the arms Player $2$ can select from. Consider the following 'Bad' event,

\begin{equation}
B:= \{\hat{\mu}^1_{(2, 1)}(t, 1)> \max_{a \neq (2, 1)}\hat{\mu}^1_a(t, 1))\}\cup \{\hat{\mu}^2_{(1, 2)}(t, 1)> \max_{a \neq (1, 2)}\hat{\mu}^2_a(t, 1))\}
\end{equation}

$B$ is the event that after taking $1$ initial sample from arms $(2, 1)$ and $(1, 2)$, Player 1 determines arm $(2, 1)$ as the optimal arm, while Player 2 determines arm $(1, 2)$ as the optimal arm. In this event Player $1$ will pull arm 2 while Player 2 will pull arm 2, thus obtaining a sample from arm $(2, 2)$. In this situation the following cases are possible,
\begin{enumerate}
    \item $\eta^1_{(2, 2)}> \eta^1_{(2, 1)}$ and $\eta^2_{(2, 2)}> \eta^2_{(2, 1)}$. 
    
    \item $\eta^1_{(2, 2)}< \eta^1_{(2, 1)}$ and $\eta^2_{(2, 2)}> \eta^2_{(2, 1)}$
    
    \item $\eta^1_{(2, 2)}> \eta^1_{(2, 1)}$ and $\eta^2_{(2, 2)}< \eta^2_{(2, 1)}$
    
    \item $\eta^1_{(2, 2)}< \eta^1_{(2, 1)}$ and $\eta^2_{(2, 2)}< \eta^2_{(2, 1)}$
\end{enumerate}

In all of these cases the next round Player 1 will still pull arm 2 and Player 2 will still pull arm 2 on their next rounds. Since $\eta^i_a$ stays the same for all $a$ except for $a = (2, 2),$ they will continue to pull arm $(2, 2)$ for the rest of the rounds. It's clear to see from the regret decomposition that 
\begin{equation}
    R_T \geq P(B)\mathbb{E}[T_{(2, 2)}]\Delta_{(2, 2)} = P(B)(T - 4)(.6)
\end{equation}

Since $P(B) > 0$, it follows that $R_T$ is asymptotically linear as desired. 

\begin{figure}
    \centering
    \includegraphics[width = .4\textwidth]{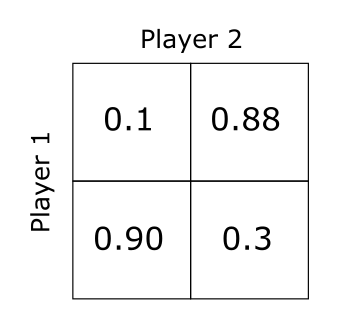}
    \caption{An $2$ player mutiplayer multi-armed stochastic bandit problem such that \texttt{mUCB} gives linear regret.}
    \label{fig:counter_ex}
\end{figure}
\end{proof}

\subsection{Problem B: Unobserved Actions, Independent Rewards}\label{sec:probCproofs}

\begin{proof} \emph{(Theorem \ref{thm:problemC})}
Decompose $R_T = R_{T, E} + R_{T, C}$, where $R_{T,E}$ is the regret incurred from the exploration sequence spaced at powers of $2$, while $R_{T, C}$ is the regret coming from committing to the arm with the highest mean. In the $\lambda$th exploration phase each arm is pulled $K(\lambda)$ times, it follows that $n_a(t) \leq K(\lfloor\log_2(t)\rfloor) \log_2(t)$. Thus, the regret decomposition given by equation \eqref{eq:regret_decomp},
\begin{equation}
    R_{T, E} \leq \sum_a K(\lfloor\log_2(t)\rfloor)\lceil \log_2(T)\rceil\Delta_a 
\end{equation}

The following claim, gives a lower bound for the number of times each arm has been explored up to time $T$
\begin{claim}
$n_a(t) \geq K_0(t)\log_2(t) $ for some function $K_0(t)$ going to infinity.
\begin{proof}
Note here that $K_0(t)$ doesn't have to map positive integers to positive integers. This statement is proved if we can show that $\lim_{t \rightarrow \infty} \frac{\sum_{i=1}^l K(\lambda)}{l} = \infty$. Suppose $\limsup_{t \geq 1}\frac{\sum_{i=1}^l K(\lambda)}{l}  = L < \infty$. Then for any $\epsilon$, there exists an integer $l$ such that $\frac{\sum_{\lambda=1}^lK(\lambda)}{l} > L - \epsilon$. We can increase $l$ so that it satisfies the condition $K(l+1) \geq K(l)+1$. Now consider the quantity $\frac{\sum_{\lambda=1}^{2l}}{2l}$ as follows
\begin{equation}
    \frac{\sum_{\lambda=1}^{2l}K(\lambda)}{2l} = \frac{\sum_{\lambda=1}^{l}(K(\lambda)+ \sum_{\lambda=l+1}^{2l}(K(\lambda)}{2l} \geq \frac{l(L - \epsilon)+ l(K(l)+1)}{2l} = \frac{L+K(l)}{2} + \frac{1 - \epsilon}{2}
\end{equation}

Since $K(l) \geq L-\epsilon$, it follows that  
\begin{equation}
       \frac{\sum_{\lambda=1}^{2l}K(\lambda)}{2l}  \geq L + \frac{1 - 2\epsilon}{2}
\end{equation}

when $\epsilon < \frac12$, we clearly have $\frac{\sum_{\lambda=1}^{2l}K(\lambda)}{2l} > L$. Since $\frac{\sum_{\lambda=1}^{2l}K(\lambda)}{2l}$ is an nondecreasing sequence, the contradiction gives the desired result. 
\end{proof}
\end{claim}

 Using equation \eqref{eq:regret_decomp}, and the assumption that $(1,...,1)$ is optimal,
\begin{equation}
    R_{T,C} \leq \sum_a \Delta_a\mathbb{E}\left(\sum_{t \in C} I\left[\hat{\mu}_{(1,...,1)}(n_{(1,...,1)}(t)) \not> \max_{a\neq (1,...,1)}\hat{\mu}_{a}(n_a(t))\right]\right)
\end{equation}

Let $\epsilon < \frac12 \min_{a: \Delta_a > 0} \Delta_a$. In this problem we consider the following good event
\begin{equation}
    G_a(i) = \{|\hat{\mu}(t, n_a(t)) - \mu_a|<\epsilon \}
\end{equation}

$G_a(i)$ is the event when the observed mean of player $i$ after $n_a(t)$ samples of arm $a$ is within $\epsilon$ of the true mean of arm $a$. For any fixed $t$ in the exploitation sequence, it follows that when $\bigcap_{a, i} G_a(i)$ occurs, the optimal arm is pulled and thus the regret is $0$. Thus,
\begin{align}
     R_{T, C} &\leq \sum_{t \in C}\sum_a P\left\{\overline{\bigcap_{i} G_a(i)}\right\}\Delta_a & \\
    &\leq  \sum_{t=1}^T\sum_a \sum_{i=1}^M P\left\{\overline{ G_a(i)}\right\}\Delta_a &\\
    &=  \sum_{t=1}^T\sum_a M P\left\{|\hat{\mu}(t, n_a(t)) - \mu_a|>\epsilon \right\}\Delta_a &\\
    &\leq \sum_{t=1}^T \sum_a 2Me^{-\frac{n_a(T)\epsilon^2}{2}}\Delta_a \quad &\text{By Lemma \ref{corollary5.5}}\\
    &\leq \sum_{t=1}^T \sum_a 2Me^{-\frac{K_0(t)\log_2(t)\epsilon^2}{2}}\Delta_a &\text{using } n_a(t) \geq K_0(t)\log_2(t) \\
     &\leq \sum_a MC'\Delta_a \sum_{t=1}^Tt^{-\frac{K_0(t)\epsilon^2}{2}}\\
     & \leq \sum_a MC'\Delta_a \sum_{t=1}^\infty t^{-\frac{K_0(t)\epsilon^2}{2}}
\end{align}

Since $f(t):=t^{-\frac{K_0(t)\epsilon^2}{2}}$ is monotonic, it is bounded by $1+ \int_1^T t^{-\frac{K_0(t)\epsilon^2}{2}}dt$. However, as $\lim_{t\rightarrow \infty}K_0(t) = \infty$, it follows that there exists an $N$ such that $t > N \implies \frac{K_0(t)\epsilon^2}{2} > 3/2 $. Since $\gamma > 1 \implies \int_1^\infty \frac{1}{t^{\gamma}}dt<\infty$, it follows that 
$\int_N^\infty t^{-\frac{K_0(t)\epsilon^2}{2}}dt<\infty$.
Thus $R_{T, C}$ is bounded below a constant less then $\infty$, so our total regret $R_T = O(K(T)\log(T))$ as desired.
\end{proof}

	\section{Numerical Experiments}\label{sec:numerical}
In this section we run numerical simulations on a three player stochastic bandit problem problem with arms $a$ from the set $a \in [2]\times[2]\times[2]$. All the plots are given in Fiure \ref{plot}. Each arm has Gaussian distribution with mean in $[.1, .9]$ and standard deviation $(0, .03]$. The means and standard deviations were chosen this way so that the probability that a reward is in the interval $[0, 1]$ is very high. The horizon is for $T = 100,000$ rounds, with the entire simulations run in the same environment for $10$ times. In Figure \ref{plot}, there are two regret verses log-time plots for Problem A (left) and Problem B (right) which compare Algorithm \texttt{mUCB} with the Agnostic Standard UCB and the disCo Algorithm in \cite{xu2015distributed} with Algorithm \texttt{mDSEE}. For each algorithm, the mean regret across all $10$ runs is plotted, with the shaded areas of the same color being $2$ standard deviations above and below the mean. This gives a $95\%$ confidence interval, and from this plot it is clear to see that a time progresses, the confidence intervals for the algorithms in comparison diverge from each other. This tells us that with high probability, one algorithm has lower regret than the other. 

For Problem A, we show the regret of the \texttt{mUCB} algorithm in black. As this curve is linear, it's clear that the regret is $O(\log(n))$ in this environment. We also compared this to the simple agnostic UCB algorithm in green. More explicitly, this is where each players treat the environment as a single player stochatic multi-armed bandit environment, and picks the arms with the highest UCB index at each round as is the case in the standard UCB algorithm. As expected, this algorithm gives linear regret.

For Problem B, the a regret verses time plot for \texttt{mDSEE} algorithm is given in black. In this setting, we set $K(\lambda)=\lambda$, and thus, from Theorem \ref{thm:problemC}, we expect the regret to be $O(\log(T)^2)$. From Figure \ref{plot} it is clear that the curve grows at most quadratically, which bolsters this Theorem. We also plotted the regret of the disCo algorithm in \cite{xu2015distributed} in red. In the beginning the regret is linear, but then it becomes log. This is because $\zeta(t) := \frac{2.01}{\Delta_{\min}}\ln(t)$, which is used to determine when to explore, is very large even for small values of $t$. Thus, there is a lot of exploration that happens in the beginning that incurs linear regret. However, since $\zeta(t)$ grows logarithmically, eventually the exploration sequences will become more sparse, thus giving us log regret. While asymptotically disCo is better, for smaller values of $t$, \texttt{mDSEE} gives much better regret as one can see in this plot.

\begin{figure}
    \centering
    \includegraphics[width = .4\textwidth]{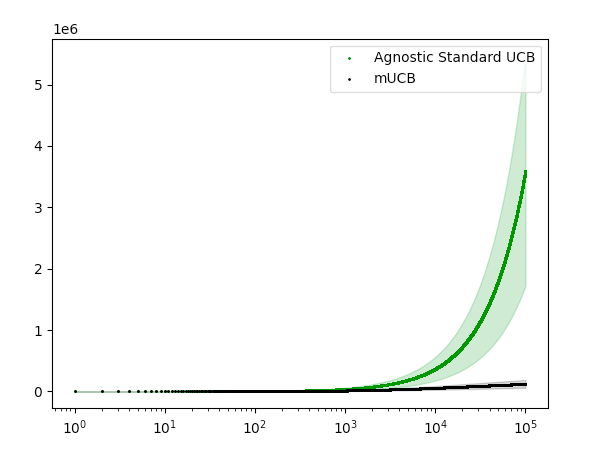}
     \includegraphics[width = .4\textwidth]{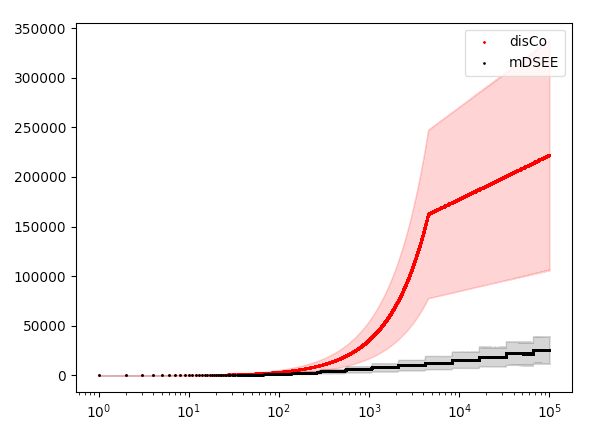} 
    \caption{Log plots of Average regret across 10 runs, with 2 standard deviations above and below the mean in the same color (95\% confidence interval). (i) Log plots of regret verses round $T$ of a multi-player stochastic bandit problem described in Problem A using Algorithm \texttt{mUCB} (black) and using the agnostic standard UCB algorithm (green). (ii) Log plots of regret verses round $T$ of a Multi-player stochastic bandit problem described in Problem B using Algorithm \texttt{mDSEE} (black) and the disCo Algorithm provided in \cite{xu2015distributed} (red)}
    \label{plot}
\end{figure}

	\section{Conclusions}
	
	In this paper, we have introduced a general framework to study decentralized online learning with team objectives in a multiplayer MAB model. If all players have the same information, this would be just like a standard single player (centralized) MAB problem. Information asymmetry is what makes the problem interesting and challenging. We introduce three types of information asymmetry: in actions, in rewards, and in both actions and rewards. We then showed that a multiplayer version of UCB algorithm is able to achieve order-optimal regret when the there is information asymmetry in actions. We then showed that we can achieve near log regret even when there is information asymmetry in both actions and rewards. Finally, we show that when there is information asymmetry in rewards, the algorithm given when there is there is information asymmetry in actions gives linear regret. For future work, considering decentralized online learning in a multiplayer MDP setting would be another interesting new direction. 
	
	
	\newpage
	\bibliographystyle{plain}
	\bibliography{references}
	\newpage
	\appendix

\section{Appendix}\label{sec:appendix}

\begin{lemma}\cite{lattimore2020bandit}
	Assume that $X_i - \mu$ are independent, $\sigma$-subgaussian random variables. Then, for any $\epsilon \geq 0$,
	\begin{equation}
		P\left(\hat{\mu} \geq \mu + \epsilon\right) \leq e^{-\frac{T\epsilon^2}{2\sigma^2}} \quad \text{and} \quad P\left(\hat{\mu} \leq \mu - \epsilon\right) \leq e^{-\frac{T\epsilon^2}{2\sigma^2}}
	\end{equation}
	where $\hat{\mu} = \frac{1}{T}\sum_{t=1}^T X_t$. 
	\label{corollary5.5}	
\end{lemma}	
\begin{proof}
	The reader is encouraged to look at Corollary 5.5 of \cite{lattimore2020bandit}. It relies on the observation that $\hat{\mu} - \mu$ is $\sigma/\sqrt{T}$-subgaussian. 
\end{proof}

\begin{proof} (\textbf{Lemma \ref{lem:regret_decomp}})
	Letting $I$ be the indicator function and $a(t)$ the action taken in round $t$, we have by definition of regret given in equation \eqref{eq:regret}.
	\begin{equation}\label{eq:regret3}
		R_T^i = \sum_{a = \times_i^M \mathcal{K}_i} \sum_{t=1}^T E[(\mu^* - X_{a(t)}^i)I\{a(t) = a\}].
	\end{equation}
	
	The expected reward in round $t$ conditioned on $a(t)$ is $\mu_{a(t)}$ and thus:
	\begin{align}
		E[(\mu^* - X_{a(t)}^i)I\{a(t)= a\}|a(t)] &= I\{a(t) = a\}E[\mu^* - X_{a(t)}^i|a(t)]\\
		&= I\{a(t) = a\}(\mu^* - \mu_{a(t)})\\
		&= I\{a(t) = a\}(\mu^* - \mu_{a})\\
		&= I\{a(t)=a\}\Delta_a.
	\end{align}
	
	We can now plug this into equation \eqref{eq:regret3} and the result follows. 
\end{proof}

\subsection{Addendum to Section \ref{sec:probAproofs}}\label{sec:appendprobAproofs}

\begin{proof} (\textbf{Lemma \ref{lem:pulls}})
	We do a proof by contradiction, supposing $n_a(T) > m_a$ while $G_a$, then arm $a$ was played more than $m_a$ times over the $T$ rounds so there must be a round $t$ such that $n_a\left(t-1\right) = m_a$ and $a(t) = a$. We now apply the definition of $G_{a}$ to obtain:
	
	\begin{align}
		\eta_{a}\left(t - 1\right) &= \hat{\mu}_{a}\left(t - 1, \cdot\right) + \sqrt{\frac{2\log\left(\frac{1}{\delta}\right)}{n_a\left(t-1\right)}}\\
		&= \hat{\mu}_{a}\left(t-1, m_a\right) +  \sqrt{\frac{2\log\left(\frac{1}{\delta}\right)}{m_a}}\\
		&< \mu_{1,...,1}\\
		&< \eta_{1,...,1}\left(t-1\right),
	\end{align}
	
	and thus $a(t) \neq a$ which gives us a contradiction.
\end{proof}

\begin{proof} (\textbf{Claim \ref{claim:1}}) Taking complement of $G_a$,
	\begin{equation}
		G_{a}^c = \left\{\mu_{\left(1,...,1\right)} \geq \min_{t \in [n]} \eta_{a}(t)\right\} \bigcup \left\{\hat{\mu}_{a}\left(\cdot, m_a\right) + \sqrt{\frac{2}{m_a}\log\left(\frac{1}{\delta}\right)} \geq \mu_{\left(1,...,1\right)} \right\}
	\end{equation}
	so that 
	\begin{equation}\label{blind.eq2}
		P\left(G_{a}^c\right)\leq P\left\{\mu_{\left(1,...,1\right)} \geq \min_{t \in [n]} \eta_{a}(t)\right\}+ P \left\{\hat{\mu}_{a}\left(\cdot, m_a\right) + \sqrt{\frac{2}{m_a}\log\left(\frac{1}{\delta}\right)} \geq \mu_{\left(1,...,1\right)} \right\}.
	\end{equation}
	
	For the first term in the equation above, we have:
	\begin{align}
		P\left\{\mu_{\left(1,...,1\right)} \geq \min_{t \in [n]} \eta_{a}(t)\right\}
		&\leq P\left(\bigcup_{s \in [n]}\left\{ \mu_{1,...,1} \geq  \hat{\mu}_{\left(1,...,1\right)}\left(\cdot, s\right) + \sqrt{\frac{2\log\left(1/\delta\right)}{s}} \right\}\right)\\
		&\leq \sum_{s = 1}^T P\left\{\mu_{1,...,1} \geq  \hat{\mu}_{a}\left(\cdot, s\right) + \sqrt{\frac{2\log\left(1/\delta\right)}{s}} \right\},
		\label{blind.eq5}
	\end{align}
	where the summation index $s$ is the number of times arm $1$ has been pulled. Since $X_{a}$ has $1$-subgaussian distribution, using Theorem \ref{corollary5.5}, it follows that $\forall s \in [1, n]$
	\begin{equation}
		P\left\{\hat{\mu}_{1,...,1}\left(\cdot, s\right) \leq u_{1,...,1} -  \sqrt{\frac{2\log\left(1/\delta\right)}{s}} \right\} \leq \delta
	\end{equation}
	
	and thus 
	\begin{equation}
		P\left\{\mu_{1,...,1} \geq \min_{t \in [n]} \eta_{a}(t)\right\} \leq \sum_{s = 1}^T \delta = T\delta.
	\end{equation}
	
	Now, let us bound the second term on the RHS in equation \eqref{blind.eq2} by first choosing $m_a$ so that
	\begin{equation}\label{blind.eq3}
		\Delta_a - \sqrt{\frac{2\log\left(1/\delta\right)}{m_a}} \geq c \Delta_a.
	\end{equation}
	Then, the term we wish to bound becomes:
	\begin{align}
		P\left\{\hat{\mu}_{i}\left(\cdot, m_a\right) + \sqrt{\frac{2}{m_a}\log\left(\frac{1}{\delta}\right)} \geq \mu_{1,...,1} \right\}  &=
		P\left\{\hat{\mu}_{i}\left(\cdot, m_a\right) - \mu_{i} \geq \Delta_a -  \sqrt{\frac{2}{m_a}\log\left(\frac{1}{\delta}\right)} \right\} \\
		& \leq e^{-\frac{m_ac^2\Delta_a^2}{2},}
	\end{align}
	where the last inequality comes from using Theorem \ref{corollary5.5} again. Combining this with inequalities \eqref{blind.eq2} and \eqref{blind.eq5} proves Claim \ref{claim:1}.
\end{proof}

\paragraph*{Reward independent Regret Bounds.}  We now present proof of Theorem \ref{thm:problemA-gapindep}.

\begin{proof} (\textbf{Theorem \ref{thm:problemA-gapindep}})
	We first take the regret decomposition given by equation \eqref{eq:regret_decomp} and partition the tuples $a$ to those whose mean are at most $\epsilon$ away from the optimal  and those whose means are more than $\epsilon$. This gives us the following inequality: 
	\begin{align}
		R_T &= \sum_a \Delta_{a}\mathbb{E}\left[n_a(T)\right] = \sum_{\Delta_{a} > \epsilon} \Delta_{a}\mathbb{E}\left[n_a(T)\right] + \sum_{\Delta_{a} \leq \epsilon} \Delta_{a}\mathbb{E}\left[n_a(T)\right]\leq \sum_{\Delta_{a} \geq \epsilon} \Delta_{a}\mathbb{E}\left[n_a(T)\right] + \epsilon T.
	\end{align}
	Following the same reasoning as in Theorem \ref{thm:problemA} yields the following inequality:
	\begin{equation}
		R_T \leq 3\sum_{\Delta_{a} > \epsilon} \Delta_{a} + \sum_{ \Delta_{a} > \epsilon}\frac{2\left(3+2\sqrt{2}\right)\log(T)}{\epsilon} + \epsilon T \leq 3ab + \sum_{ \Delta_{a} > \epsilon}\frac{2\left(3+2\sqrt{2}\right)\log(T)}{\epsilon} + \epsilon T. \label{blind.eq6}
	\end{equation}
	Since inequality \eqref{blind.eq6} holds for all $\epsilon$, we can pick $\epsilon = \sqrt{\frac{\log(T)}{T}}$ to obtain result in Theorem \ref{thm:problemA-gapindep}.
\end{proof}

\paragraph*{Markovian Rewards in  Problems A.}\label{sec:probAmarkov-proofs}

We now consider the rewards to be Markovian. Note that the regret is as defined in equation \eqref{eq:regret} and we use is Algorithm \ref{algo:mucb}. The optimal arm is measured relative to the optimal stationary Markov distribution rather than the arm with the highest mean.  With this definition of regret the regret decomposition as stated in equation \eqref{eq:regret_decomp} still holds and will be used to derive an upper bound for the regret. 

Theorem \ref{thm:lezaud98} is needed below which is restated for completeness.
\begin{theorem}\cite{lezaud1998chernoff}\label{thm:lezaud98}
	Let $X_t, t \geq 1$ be an irreducible, aperiodic Markov chain on a finite state space $\mathcal{X}$ with transition probability matrix $P$, an initial distriubtion $\lambda$ and a stationary distribution $\pi$. Denote $N_\lambda = \left\lvert \left\lvert \left(\frac{\lambda_x}{\pi_x}, x \in \mathcal{X}\right)\right\rvert\right\rvert$. Let $\rho$ be the eigenvalue gap, $1 - \lambda_2$, where $\lambda_2$ is the second largest eigenvalue of he matrix $\Tilde{P}$, where $\Tilde{P}$ is the multiplicative symmetrization of the transition matrix $P$. Let $f:\mathcal{X} \rightarrow \mathbb{R}$ be such that $\sum_{x \in \mathcal{X}}\pi_x f\left(x\right) = 0, ||f||_\infty \leq 1, ||f||_2^2 \leq 1$. If $\Tilde{P}$ is irreducible, then for any $\gamma>0$, $P\left(\sum_{a = 1}^T \frac{X_a}{t} \geq \gamma\right) \leq N_\lambda e^{-\frac{t\rho\gamma^2}{28}}$.  
	\begin{proof}
		The reader is encouraged to look at proof of Theorem 1.1 in \cite{lezaud1998chernoff}
	\end{proof}
\end{theorem}

\begin{proof} (\textbf{Theorem \ref{thm:problemA-markov}})
	Suppose without loss of any generality that the arm $\left(1,...,1\right)$ is optimal (has the highest expected reward). We can consider the ``good'' event as defined by equation \eqref{eq_good}, and decompose it as in \eqref{blind.eq2}. We bound the term $P \left\{\hat{\mu}_{a}\left(\cdot, m_a\right) + \sqrt{\frac{2}{m_a}\log\left(\frac{1}{\delta}\right)} \geq \mu_{\left(1,...,1\right)} \right\}$ using Theorem \ref{thm:lezaud98} with $\gamma = \mu_{\left(1,...,1\right)} - \sqrt{\frac{2}{m_a}\log\left(\frac{1}{\delta}\right)}$ to obtain
	\begin{equation}\label{Markovian.eq1}
		P \left\{\hat{\mu}_{a}\left(\cdot, m_a\right) + \sqrt{\frac{2}{m_a}\log\left(\frac{1}{\delta}\right)} \geq \mu_{\left(1,...,1\right)} \right\} \leq C\delta^{\frac{K_2 T}{m_a}}e^{K_2T\sqrt{\frac{\log\left(\frac{1}{\delta}\right)}{m_a}}}
	\end{equation}
	for some constants $C, K_2, k_2 > 0$. For the other term $P\left\{\mu_{\left(1,...,1\right)} \geq \min_{t \in [n]} \eta_{a}(t)\right\}$, we can use equation \eqref{blind.eq2} and just bound:
	\begin{equation}
		\sum_{s = 1}^T P\left\{\mu_{1,...,1} \geq  \hat{\mu}_{a}\left(\cdot, s\right) + \sqrt{\frac{2\log\left(1/\delta\right)}{s}} \right\} = \sum_{s = 1}^T P\left\{\hat{\mu}_{a}\left(\cdot, s\right) \leq \mu_{1,...,1}  -\sqrt{\frac{2\log\left(1/\delta\right)}{s}} \right\}.
	\end{equation}
	
	Using $\gamma = \mu_{1,...,1} - \sqrt{\frac{2\log\left(1/\delta\right)}{s}}$, we can use Theorem \ref{thm:lezaud98} again to obtain: 
	\begin{equation}\label{Markovian.eq2}
		P\left\{\mu_{1,...,1} \geq  \hat{\mu}_{a}\left(\cdot, s\right) + \sqrt{\frac{2\log\left(1/\delta\right)}{s}} \right\} \leq C'\delta^{\frac{K'_1 T}{m_a}}e^{K'_2T\sqrt{\frac{\log\left(\frac{1}{\delta}\right)}{s}}} \leq C'\delta^{\frac{K'_1 T}{m_a}}e^{K'_2T\sqrt{\frac{\log\left(\frac{1}{\delta}\right)}{T}}} 
	\end{equation}
	for some constants $C', k'_1, k'_2 > 0$. The last inequality comes from $s \in [T]$. Thus, using equations \eqref{Markovian.eq1} and \eqref{Markovian.eq2} and plugging it into equation \eqref{eq:expectation} yields:
	
	\begin{equation}
		\mathbb{E}\left[n_a(T)\right] \leq m_a + C\delta^{\frac{K_2 T}{m_a}}e^{K_2T\sqrt{\frac{\log\left(\frac{1}{\delta}\right)}{m_a}}}+ C'\delta^{\frac{K'_1 T}{m_a}}e^{K'_2T\sqrt{\frac{\log\left(\frac{1}{\delta}\right)}{T}}}
	\end{equation}
	
	Letting $\delta = \frac{1}{n^\alpha}$ for some $\alpha > 0$, we just need $\alpha$ and $m_a$ to satisfy:
	\begin{align}
		&-\frac{\alpha K_2 T}{m_a} + k_2 T\sqrt{\frac{\alpha}{\log(T) m_a}} < 0\\
		&-\frac{\alpha k'_1 T}{m_a} + k'_2 T\sqrt{\frac{\alpha}{\log(T) m_a}} < 0
	\end{align}
	
	for all sufficiently large $T$. If we let $m_a= \log\left(\frac{1}{\delta}\right)$, then letting $\alpha > \left(\frac{K_2}{K_2}\right)^2$ is sufficient. Thus, with this value of $\alpha$, we now have the inequality,
	\begin{equation}\label{Markovian.eq3}
		\mathbb{E}\left[n_a(T)\right] \leq 2C' + \alpha \log(T). 
	\end{equation}
	Plugging this into \eqref{eq:regret_decomp} we obtain the regret bound in Theorem \ref{thm:problemA-markov}.
\end{proof}

\end{document}